\documentclass[10pt,twocolumn,letterpaper]{article}

\usepackage{cvpr}
\usepackage{times}
\usepackage{epsfig}
\usepackage{graphicx}
\usepackage{amsmath}
\usepackage{amssymb}
\usepackage{bm}
\usepackage[dvipsnames]{xcolor}
\usepackage{tikz,subfigure,sectsty}
\usepackage[noend]{algpseudocode}
\usepackage{amsthm}
\usepackage{cite}

\newtheorem{prop}{Proposition}

\newcommand{\tens}[1]{%
  \mathbin{\mathop{\otimes}\limits_{#1}}%
}

\newcommand{\bx}{{\bm x}}

\definecolor{ballblue}{rgb}{0.13, 0.67, 0.8}
\definecolor{newbrown}{rgb}{0.65, 0.16, 0.16}	
\definecolor{cblue}{rgb}{0.64, 0.76, 0.68}
\definecolor{cerise}{rgb}{0.87, 0.19, 0.39}
\definecolor{jasper}{rgb}{0.84, 0.23, 0.24}
\definecolor{mblue}{rgb}{0.15, 0.38, 0.61}
\definecolor{ballblue}{rgb}{0.13, 0.67, 0.8}
\definecolor{newbrown}{rgb}{0.65, 0.16, 0.16}	
\definecolor{cblue}{rgb}{0.64, 0.76, 0.68}
\definecolor{cerise}{rgb}{0.87, 0.19, 0.39}
\definecolor{jasper}{rgb}{0.84, 0.23, 0.24}
\definecolor{mblue}{rgb}{0.15, 0.38, 0.61}
\definecolor{asparagus}{rgb}{0.53, 0.66, 0.42}
\definecolor{carrotorange}{rgb}{0.93, 0.57, 0.13}
		
\definecolor{mycolor1}{HTML}{00e600}
\definecolor{mycolor2}{HTML}{990000}
\definecolor{g}{HTML}{009900}
\definecolor{r}{HTML}{990000}
\definecolor{b}{HTML}{000099}
\usetikzlibrary{decorations.pathmorphing}
\usetikzlibrary{patterns}
\usetikzlibrary{decorations.pathmorphing}
    
\usepackage{pgfplots}
\pgfplotsset{width=7cm,compat=1.5}
\usepackage{authblk}

\newcommand{\myparagraph}[1]{\vspace{0.4em}\noindent\textbf{#1}}

\usepackage[pagebackref=true,breaklinks=true,letterpaper=true,colorlinks,bookmarks=false]{hyperref}

\cvprfinalcopy 

\def\cvprPaperID{5661} 

\ifcvprfinal\fi
\begin{document}

\title{Local Temporal Bilinear Pooling for Fine-Grained Action Parsing}

\author[1,2]{Yan Zhang}
\author[2,3]{Siyu Tang}
\author[2]{Krikamol Muandet}
\author[1]{Christian Jarvers}
\author[1]{Heiko Neumann}
\affil[1]{Institute of Neural Information Processing, Ulm University, Ulm, Germany 
}
\affil[2]{Max Planck Institute for Intelligent Systems, T\"{u}bingen, Germany
}
\affil[3]{University of T\"{u}bingen, T\"{u}bingen, Germany}

\maketitle
\thispagestyle{empty}

\begin{abstract}
   Fine-grained temporal action parsing is important in many applications, such as daily activity understanding, human motion analysis, surgical robotics and others requiring subtle and precise operations over a long-term period. In this paper we propose a novel bilinear pooling operation, which is used in intermediate layers of a temporal convolutional encoder-decoder net. In contrast to previous work, our proposed bilinear pooling is learnable and hence can capture more complex local statistics than the conventional counterpart. In addition, we introduce {\em exact} lower-dimension representations of our bilinear forms, so that the dimensionality is reduced without suffering from information loss nor requiring extra computation. We perform extensive experiments to quantitatively analyze our model and show the superior performances to other state-of-the-art pooling work on various datasets. 
\end{abstract}

\section{Introduction}
\label{sec:introduction}
Parsing fine-grained actions over time is important in many applications, which require understanding of subtle and precise operations over long-term periods, e.g. daily activities \cite{li2015delving}, surgical robots \cite{ahmidi2017dataset}, human motion analysis \cite{zhang2018temporal} and animal behavior analysis in the lab \cite{mathis2018deeplabcut}.
Given a video or a generic time sequence of feature vectors, an action parsing algorithm aims at assigning each frame an action label, such that the entire sequence is partitioned into several disjoint semantic action primitives. Thus, tasks of action recognition, temporal semantic segmentation and action detection in untrimmed videos can be solved in one framework.

Recently, fine-grained action parsing algorithms based on deep convolutional nets are highly effective. For example, the method proposed in \cite{lea2016segmental} and \cite{lea_2017_cvpr} first extracts frame-wise feature vectors via a spatial convolutional net, and then assigns action labels to individual frames via a temporal convolutional encoder-decoder (TCED) architecture. 
As reported, such TCED net outperforms other methods on challenging fine-grained action datasets of various scenarios.


While being straightforward, a notable caveat of the TCED architecture in \cite{lea_2017_cvpr} is that the max pooling operation embedded between convolutional layers in the encoder ignores high-order temporal structures, and hence cannot differentiate two fine-grained actions with identical first-order but different second-order statistics. 
Taking grasping object by hand as an example, when the feature vector of each frame is the concatenation of 3D positions of the finger tips, max pooling on several consecutive frames yields the hand position, and hence tells where to grasp the object. In parallel, the second-order information can indicate the finger scatter, and hence tells how to grasp the object. Thus, different orders of information are rather independent and complementary to precisely describe an action. Without the second-order information, it is hardly able to distinguish whether to grasp a coin or a book at the same position. 
Motivated by this example, as well as several recent studies showing that bilinear pooling outperforms first-order pooling on fine-grained tasks (e.g. \cite{carreira2012semantic,gao2016compact,yu2018statistically}), we aim at introducing bilinear pooling into the TCED net, so that second-order statistics can be incorporated to produce better fine-grained action parsing results. 
We refer to Sec. \ref{sec:max_vs_bi} for detailed analysis of the benefits of second-order information.


However, combining such two methods is highly non-trivial, which requires to overcome drawbacks of conventional bilinear pooling: (1) The conventional bilinear pooling is designed for visual classification. Thus, it aggregates all the features globally, destroying the local data structure which is important for semantic segmentation. (2) The conventional bilinear pooling aggregates the outer products of the feature vectors by averaging, and hence loses representativeness when the real data distribution is complex. (3) The conventional bilinear pooling lifts the feature dimension from $d$ to $d^2$, causing parameter proliferation in the neural net and expensive computational cost.

In this work we extend the conventional bilinear pooling from several aspects and make it suitable for fine-grained action parsing. Specifically, we make the following
contributions:
(1) To enrich the representativeness, we decouple the first and second-order components from the bilinear form, and replace the averaging by convolution of a learnable filter. In this case, the proposed bilinear form is adaptive to the data and 
guided by the training objective.
(2) To reduce the dimensionality {\em without} suffering from information loss or requiring extra computation, we propose lower-dimensional feature mappings than the explicit bilinear compositions. Such feature mapping is equivalent to the bilinear form, in the sense that the associated kernel function, and hence the {\em reproducing-kernel Hilbert space} (RKHS), is identical.
(3) We perform extensive experiments to investigate our novel bilinear pooling methods, and show that the proposed method consistently improves or is on-par with the performance of the state-of-the-art methods on diverse datasets.
To our knowledge, we are the first to employ bilinear pooling in a convolutional encoder-decoder architecture for fine-grained action parsing over time.

\section{Related work}
\label{sec:related_work}

\myparagraph{Fine-grained temporal action parsing.}
\cite{fathi2013modeling} proposes to learn object and material states, and partition actions by detecting the state transitions. 
\cite{richard2016temporal} applies a statistical language model to capture action temporal dynamics. 
\cite{singh2016first} proposes an Ego-ConvNet incorporating two streams for extracting spatial features and spatiotemporal features respectively from pre-defined video segments. The results are improved when combining Fisher vectors \cite{6751336} from spatial and optical flow descriptors \cite{wang2015action}. 
\cite{singh2016multi} proposes a multi-modal bidirectional LSTM model to generate a label sequence of a video to incorporate forward and backward temporal dynamics. 
\cite{lea2016learning} proposes a conditional random field with skip connections in the temporal domain and starting-and-ending frame priors, which is learned via a structured support vector machine. \cite{lea2016segmental} proposes a multi-modal deep neural net with the similar structure of the VGG net. After training and extracting frame-wise features, a temporal convolutional net and a semi-Markov conditional random field are applied to produce the final segmentation result. Based on the spatial features from \cite{lea2016segmental}, \cite{lea_2017_cvpr} proposes two kinds of temporal convolutional networks with the encoder-decoder architecture. The first net comprises layers of convolution and max pooling; the second net uses dilated temporal convolution and skipped connections to capture long-range temporal structures. Our work uses the temporal encoder-decoder architecture proposed by \cite{lea_2017_cvpr}. To capture second-order statistics, we replace the max pooling in \cite{lea_2017_cvpr} by our proposed bilinear pooling operations. We compare our method with others in Sec. \ref{sec:experiment}. Although more complicated architectures, e.g. \cite{Lei_2018_CVPR} \cite{mac2018locally}, can also improve the performance, our work focuses on the pooling operation and hence investigating more advanced architectures is out of our scope.

\myparagraph{Bilinear pooling.} Bilinear pooling (or second-order pooling) is widely used in fine-grained visual classification \cite{carreira2012semantic,me_tensor_tech_rep,moghimi2016boosted,li2017second,lin2015bilinear,lin2018bilinear,lin2017improved,lin2018second,gao2016compact,personreid:eccv:2018,kong2017low,li2017towards,wei2018grassmann,yu2018hierarchical, yu2018statistically,wang2017g2denet,gou2018monet,simon2017generalized}, visual questioning answering \cite{fukui2016multimodal,kim2016hadamard,yu2017multi}, feature fusion and disentangling \cite{lin2015bilinear, lin2018bilinear,tenenbaum2000separating,diba2017deep,feichtenhofer2016convolutional,hu2018deep}, action recognition \cite{yue2018compact,wang2017non,girdhar2017attentional,hu2018deep,feichtenhofer2016convolutional, cherian2017higher} and other tasks. In deep neural nets, bilinear pooling is mostly used only once before the classification layer, e.g. in \cite{lin2015bilinear,lin2018bilinear,lin2018second,gao2016compact,personreid:eccv:2018,diba2017deep,feichtenhofer2016convolutional,yu2018statistically,wang2017g2denet,simon2017generalized}, or embedded within the classifier, e.g. in \cite{wei2018grassmann,kong2017low}.

There are three major research directions regarding bilinear pooling: (1) Dimension reduction while minimizing information loss. \cite{gao2016compact,gou2018monet,yue2018compact} use tensor sketch \cite{pham2013fast} to reduce the dimension of vectorized bilinear forms. The studies of \cite{lin2018bilinear,yu2018statistically} use parametric dimension reduction approaches, which can be learned via back-propagation. The work in \cite{kim2016hadamard} \cite{yu2018hierarchical} finds a low-rank approximation of the bilinear forms, so as to convert vector outer product into Hadamard multiplication for cheap computation. \cite{wei2018grassmann,me_tensor_tech_rep,lin2017improved} utilize singular value decomposition (SVD), which can be used to select principle components and increase the performance at a higher computational cost. (2) Multiple bilinear pooling layers in deep neural nets. \cite{hu2018deep} factorizes bilinear composition into consecutive matrix multiplications along different dimensions. \cite{yu2018hierarchical} uses the low-rank approximation as in \cite{kim2016hadamard}, and aggregates features hierarchically. \cite{dai2017fason} fuses first and second-order information across layers to improve texture recognition. (3) Methods to capture richer feature statistics, so that more complex distributions can be represented.
\cite{koniusz2017higher} proposes a higher-order pooling scheme to extract feature co-occurrences of visual words based on linearization of a higher-order polynomial kernel. \cite{cui2017kernel} applies tensor sketch to generate a compact explicit feature map up to p-th order. Despite increasing the representativeness, more computational loads are caused. \cite{cherian2017higher} linearizes a Gaussian kernel to derive a higher-order descriptor from the late fusion of CNN classifier scores for action recognition.

The novelties of our bilinear pooling method contribute to all the three research directions. First, we prove that our proposed bilinear forms {correspond to feature mappings of some reproducing-kernel Hilbert spaces (RKHSs) endowed with polynomial kernels. We then find \emph{exact} lower-dimensional alternative feature representations that retain the kernel evaluations in these RKHSs. As a result, the dimension can be reduced {\em without} information loss and additional computation.} Second, our bilinear forms are used in multiple layers in the temporal convolutional encoder-decoder architecture, instead of being only used at the network top. Third, the first and second-order components of the bilinear forms can be decoupled and each of them has different {\em learnable} weights. Despite staying in second-order, the learnable weights enable to create adaptive local statistics to the data, and hence can capture more complex statistics than the conventional bilinear pooling.

\section{Method}
\label{sec:method}

\subsection{Preliminaries}

\myparagraph{Temporal Convolutional Encoder-Decoder.}
The TCED net takes a temporal sequence of feature vectors and assigns an action label to each input feature vector. It comprises a stack of encoders and decoders, and a fully connected module to generate frame-wise action labels. Each encoder comprises a 1D temporal convolution layer with an activation function and a pooling layer to extract local statistics. 
After each encoder, the temporal resolution is halved. The decoder has a symmetric structure with the encoder, composed of a 1D temporal convolution layer and a upsampling layer to perform nearest-neighbor interpolation. After each decoder, the temporal resolution is doubled. The fully connected module incorporates a time-distributed fully connected layer to perform linear transformation at each time instant. Then each output is passed to a softmax function to fit the ground truth one-hot encoded action label. We refer to \cite[Figure 1]{lea_2017_cvpr} for details. 


\myparagraph{Bilinear Pooling.}
Given a set of generic feature vectors with $\bx \in \mathcal{X}$, the conventional bilinear pooling \cite{carreira2012semantic,li2017second,lin2015bilinear,lin2018bilinear} can be given by

\begin{equation}
\label{eq:verybasic}
    \mathcal{B(\mathcal{X}}) =vec \left( \frac{1}{|\mathcal{X}|}  \sum_{\bx \in \mathcal{X}} \bx \tens{} \bx \right),
\end{equation}
where $\tens{}$ denotes the vector outer product, $|\cdot|$ denotes the cardinality of the feature set and $vec(\cdot)$ denotes tensor vectorization. In this case, the bilinear composition gives a description of the feature set incorporating feature channel correlations. 

\subsection{Local Temporal Bilinear Composition}
\label{sec:method1}

In contrast to many studies that perform global pooling for visual classification, we define the feature set in Eq. \eqref{eq:verybasic} as a local temporal neighborhood set to preserve the temporal structure.
Specifically, given a temporal sequence of features {${\bm X} = \{\bx_1,...,\bx_T \}$ with $\bx_t \in \mathbb{R}^{d}$ for $t\in{1,2,...,T}$}, the local temporal bilinear composition which {\em couples} the first and second-order information is given by

\begin{equation}
\mathcal{B}_c({\bm x}_t) =  vec\left( \frac{1}{|\mathcal{N}(t)|} \cdot \sum_{\tau \in \mathcal{N}(t)} {\bm x}_{\tau} \tens{} {\bm x}_{\tau} \right),
\label{eq:bilinear_standard}
\end{equation}
where $\mathcal{N}(t)$ denotes the local temporal neighborhood set centered at time $t$. As the averaging operation ignores the real distribution in $\mathcal{N}(t)$, we enrich the representativeness of bilinear pooling with the following two perspectives.

%

\subsubsection{Decoupling First and Second-order Information}
Inspired by a physical fact that the position and the velocity of an object in motion can indicate the dynamic state independently and complementarily, we consider to separate first and second-order components from the bilinear form to describe the action via separate attributes.
Provided the feature time sequence $\{\bx_1,...,\bx_T\}$, the first-order component ${\bm \mu}$, the second-order component ${\bm \Sigma}$, and the decoupled bilinear form $\mathcal{B}_d(\cdot)$ are given by

\begin{align}
\label{eq:decoupled_bilinear}
&{\bm \mu}_t = \frac{1}{|\mathcal{N}(t)|} \cdot \sum_{\tau \in \mathcal{N}(t)}{\bm x}_{\tau}, \\
&{\bm \Sigma}_t = \frac{1}{|\mathcal{N}(t)|}  \cdot \sum_{\tau \in \mathcal{N}(t)} ({\bm x}_{\tau}-{\bm \mu}_t) \tens{}  ({\bm x}_{\tau}-{\bm \mu}_t) \text{ and} \\
&\mathcal{B}_d({\bm x}_t) = \Big( {\bm \mu}_t^T, vec\big({\bm \Sigma}_t \big) \Big)^T,
\end{align}
in which one can note $\mathcal{B}_d({\bm x}_t) \in \mathbb{R}^{d(d+1)}$. Since the first-order component is equivalent to the mean and the second-order component is equivalent to the covariance, such decomposed bilinear form can precisely describe a Gaussian distribution.

\subsubsection{Adapting local statistics to data}
When the local statistics is more complex than Gaussian distribution, only using mean and covariance is not sufficient. Rather than applying higher-order statistics (e.g. \cite{koniusz2017higher, cui2017kernel}), we consider statistics up to the second-order to retain a low computational load. Since the averaging operation in Eq. \eqref{eq:bilinear_standard} and Eq. \eqref{eq:decoupled_bilinear} can be regarded as convolution by a box filter, we generalize it to convolution by a learnable filter. Thus, the local statistics is adaptive to the data and the network objective. 
Specifically, for the coupled bilinear form the learnable version is given by

\begin{equation}
\mathcal{B}_c({\bm x}_t) =  vec\left(\sum_{\tau \in \mathcal{N}(t)} \omega_{\tau}{\bm x}_{\tau} \tens{} {\bm x}_{\tau} \right),
\label{eq:bilinear1}
\end{equation}
where the filter weights $\{ \omega_{\tau} \}$ are shared by all temporal neighbor sets, i.e. $\mathcal{N}(t)$ with $t=1,2,...,T$ .

For the decoupled bilinear form, the learnable version is given by
\begin{align}
&{\bm \mu}_t = \sum_{\tau \in \mathcal{N}(t)} p_{\tau}{\bm x}_{\tau}, \\
&{\bm \Sigma}_t = \sum_{\tau \in \mathcal{N}(t)} q_{\tau}({\bm x}_{\tau}-{\bm \mu}_t) \tens{} ({\bm x}_{\tau}-{\bm \mu}_t) \text{ and} \\
&\mathcal{B}_d({\bm x}_t) = \Big( {\bm \mu}_t^T, vec\big({\bm \Sigma}_t \big) \Big)^T,
\label{eq:bilinear2}
\end{align}
where the filter weights $\{ p_{\tau} \}$ and $\{ q_{\tau} \}$ are shared by all temporal neighbor sets. 


\subsection{Normalization}
\label{sec:normalization}
Our bilinear forms are applied in several intermediate layers of the neural net. Due to the vector outer product, small values become smaller and large values become larger as the data flows from the net bottom to top, leading to diverging spectra in the bilinear forms and very sparse features before the final classification layer. 
Here we present three normalization methods that can constrain the bilinear form spectra or densify the features.

\myparagraph{$l_2$ normalization.} We can apply $l_2$ normalization after each bilinear pooling. Since the $l_2$ norm of a vectorized matrix is equivalent to its Frobenius norm and also equivalent to the Frobenius norm of the singular value matrix after SVD, the $l_2$ normalization on the vectorized bilinear form can constrain the matrix spectra between 0 and 1, and hence eliminates the diverging spectra problem. 



\myparagraph{Regularized power normalization.} When using element-wise power normalization, as e.g. in \cite{carreira2012semantic,me_tensor_tech_rep,lin2017improved}, or matrix \cite{me_tensor_tech_rep, lin2017improved} or higher-order tensor \cite{me_tensor_tech_rep} spectral power normalization in intermediate layers of a neural net, gradients tend to explode during back-propagation when small or zero values are encountered.
We propose a regularized version and use it as an activation function after each 1D convolution layer, so that features in the net are always densified. The formula is given by

\begin{equation}
\label{eq:rpn}
\sigma(x) = \text{\bf RPN} (x) = \text{sign}(x) \cdot \left( \sqrt{|x|+\theta^2} - \sqrt{\theta^2} \right),
\end{equation}
where $\text{\bf RPN}$ stands for {\em regularized power normalization} and $\theta$ is a learnable parameter. 
As $\theta \to 0$, the {\bf RPN} function converges to the standard power normalization. There exist many studies to make power normalization well-behaved, yet detailed discussion on such topic is out of our scope. One can see \cite{koniusz2018deeper} for other smooth power normalization methods which are proposed for deep neural nets.

\myparagraph{Normalized ReLU.} \cite{lea_2017_cvpr} proposes a normalized ReLU activation function, which allows fast convergence and yields superior results to other activation functions. The formula is given by

\begin{equation}
\label{eq:nrelu}
\sigma(\bx) = \text{\bf NReLU}(\bx) = \frac{ReLU(\bx)}{max(ReLU(\bx)) + \epsilon},
\end{equation}
where $\text{\bf NReLU}$ stands for normalized ReLU, $\epsilon$ is a small positive constant and the $max(\cdot)$ operation selects the maximal value in each feature vector. Since the Frobenius norm is bounded by the max norm of a matrix \cite{golub2012matrix}, {\bf NReLU} is also able to constrain the matrix singular values and hence eliminates the diverging spectra issue. Nevertheless, it can lead to sparse features.

\subsection{Low-dimensional Representation}
\label{sec:dim_reduce}
Given an arbitrary feature vector sequence, the bilinear forms $\mathcal{B}_c$ and $\mathcal{B}_d$ can capture local temporal statistics which are adaptive to the data. However, the feature dimension is considerably increased. Specifically, given ${\bm x}_t \in \mathbb{R}^d$, we have $\mathcal{B}_c( {\bm x}_t  ) \in \mathbb{R}^{d^2} $ and $\mathcal{B}_d( {\bm x}_t  ) \in \mathbb{R}^{d(d+1)} $. To address such issue, we propose alternative lower-dimensional representations to the explicit bilinear forms defined in Eq. \eqref{eq:bilinear1} and Eq. \eqref{eq:bilinear2}. Comparing to other dimension reduction methods introduced in Sec. \ref{sec:related_work}, {our method is \emph{exact} which means it introduces \emph{neither} information loss as in approximation methods nor additional computational costs as in SVD.}

We first show that  $\mathcal{B}_c( \cdot  )$ and $\mathcal{B}_d( \cdot  )$ are feature mappings associated with reproducing kernel Hilbert spaces (RKHSs) \cite{Scholkopf01:LKS}, for which the kernels are seconds-order homogeneous and inhomogeneous polynomials, respectively. Such property can be extended to arbitrary $p$-th order polynomials. One can see more details in \cite[Chapter 3]{Muandet17:KME}.

\begin{prop}
\label{prop:inner_product}

Given $\{\bx_1,...,\bx_T\}$, we have
\begin{equation}
\left\langle \mathcal{B}_c({\bm x_{i}}), \mathcal{B}_c({\bm x_{j}}) \right\rangle_{\mathbb{R}^{d^2}} = \sum_{\tau \in \mathcal{N}(i)} \sum_{\tau' \in \mathcal{N}(j)}\omega_{\tau} \omega_{\tau'}  \left\langle  {\bm x}_{\tau},  {\bm x}_{\tau'}  \right\rangle_{\mathbb{R}^d}^2, 
\end{equation}

and 

\begin{equation}
\label{eq:decouple_form}
\begin{split}
&\left\langle \mathcal{B}_d({\bm x_{i}}), \mathcal{B}_d({\bm x_{j}}) \right\rangle_{\mathbb{R}^{d(d+1)}} = \langle {\bm \mu}_i,{\bm \mu}_j \rangle_{\mathbb{R}^d} \\
&+ \sum_{\tau \in \mathcal{N}(i)} \sum_{\tau' \in \mathcal{N}(j)} q_{\tau}  q_{\tau'}  \left\langle  {\bm x}_{\tau}-{\bm \mu}_i,  {\bm x}_{\tau'}-{\bm \mu}_j  \right\rangle_{\mathbb{R}^{d^2}}^2, 
\end{split}
\end{equation}
in which the notations are referred to the definitions in Eq. \eqref{eq:bilinear1} and Eq. \eqref{eq:bilinear2}.

\end{prop}

%
%

\begin{proof}
For the coupled bilinear composition, we have
\begin{equation}
\label{eq:proof1}
\begin{split}
&\left\langle \mathcal{B}_c({\bm x_{i}}), \mathcal{B}_c({\bm x_{j}}) \right\rangle_{\mathbb{R}^{d^2}} \\
&=\left\langle \sum_{\tau \in \mathcal{N}(i)} vec\left( \omega_{\tau}{\bm x}_{\tau}  \tens{} {\bm x}_{\tau} \right), 
 \sum_{\tau' \in \mathcal{N}(j)}  vec\left(\omega_{\tau'}{\bm x}_{\tau'} \tens{} {\bm x}_{\tau'} \right) \right\rangle \\
 &= \sum_{\tau \in \mathcal{N}(i)} \sum_{\tau' \in \mathcal{N}(j)} \omega_{\tau} \omega_{\tau'}   \left\langle vec\left({\bm x}_{\tau} \tens{} {\bm x}_{\tau} \right), vec\left({\bm x}_{\tau'} \tens{} {\bm x}_{\tau'} \right) \right\rangle \\
 &= \sum_{\tau \in \mathcal{N}(i)} \sum_{\tau' \in \mathcal{N}(j)}\omega_{\tau} \omega_{\tau'}  \left\langle  {\bm x}_{\tau},  {\bm x}_{\tau'}  \right\rangle_{\mathbb{R}^d}^2.
\end{split}
\end{equation}

For the decoupled bilinear composition, we have
\begin{align}
\left\langle \mathcal{B}_d({\bm x_{i}}), \mathcal{B}_d({\bm x_{j}}) \right\rangle_{\mathbb{R}^{d(d+1)}} &= \langle {\bm \mu}_i,{\bm \mu}_j \rangle_{\mathbb{R}^d} \nonumber \\
& +  \langle vec({\bm \Sigma}_i), vec({\bm \Sigma}_j) \rangle_{\mathbb{R}^{d^2}},
\end{align}
and hence can obtain Eq. \eqref{eq:decouple_form} following the derivation in Eq. \eqref{eq:proof1}.
\end{proof}

{
One can see from Proposition \ref{prop:inner_product} that the inner product defined w.r.t. $\mathcal{B}_c(\cdot)$ can be expressed in terms of the 2nd-degree homogeneous polynomial kernel $k(\mathbf{x},\mathbf{x}') = \langle \mathbf{x},\mathbf{x}'\rangle^2$. In general, the dimension of $\mathcal{B}_c(\cdot)$ increases \emph{exponentially} with the degree of the polynomial kernel, making it less practical when used explicitly in a deep neural net. Motivated by the fact that for a specific kernel $k(\cdot, \cdot)$, the associated feature mapping $\phi: {\bm X} \to \mathcal{H}$ is not unique, we derive a feature mapping that corresponds to the same kernel as $\mathcal{B}_c(\cdot)$, but has lower dimension. The proposed method reduces the number of parameters to be learned without sacrificing the representativeness. In particular, we show that:
}

\begin{prop}\label{prop2}
Let $\mathcal{B}_c(\bx) \in \mathbb{R}^{d^2}$ be the bilinear composition and $\phi_c(\bx) \in \mathbb{R}^{\frac{d(d+1)}{2}}$ a feature mapping defined by
\begin{equation}
\phi_c(\bx)=( \underbrace{x_1^2,...,x_d^2}_{d \text{ terms}}, \underbrace{\sqrt{2}x_1x_2, \sqrt{2}x_1x_3,...,\sqrt{2}x_{d-1}x_d}_{C(d,2) \text{ terms}} )^T .
\end{equation}
Then, it follows that for any $\bx,{\bx}'\in\mathbb{R}^d$, $$\langle \mathcal{B}_c(\bx),\mathcal{B}_c({\bx}') \rangle _{\mathbb{R}^{d^2}} = \langle \phi_c(\bx), \phi_c({\bx}') \rangle_{\mathbb{R}^{\frac{d(d+1)}{2}}}.$$

Equivalently, the second-order component defined in Eq. \eqref{eq:bilinear2} has a lower-dimensional alternative, so that $\mathcal{B}_d(\bx) \in \mathbb{R}^{d(d+1)}$ can be replaced by $\phi_d(\bx) \in \mathbb{R}^{\frac{d(d+3)}{2}}$.

\end{prop}
Due to the commutative property of tensor product, the above proposition can be proved by expanding the polynomials in Eq. \eqref{eq:proof1} and combining equivalent terms. 


Proposition \ref{prop2} shows that $\mathcal{B}_c(\cdot)$ and $\phi_c(\cdot)$ are equivalent in the sense that the corresponding kernel is the same. The advantage of using $\phi_c(\cdot)$ instead of $\mathcal{B}_c(\cdot)$ is that it has much lower dimension. For example, if each feature vector in the input sequence is 128-dimensional, $\mathcal{B}_c$ is 16384-dimensional and $\mathcal{B}_d$ is 16512-dimensional. On the other hand, the alternative feature representations $\phi_c$ is 8256-dimensional and $\phi_d$ is 8384-dimensional, approximately halving the dimensionality without losing information and without introducing extra computation.

\section{Experiment}
\label{sec:experiment}

\subsection{Datasets}
In our experiments, the input feature time sequence to the TCED net is extracted from RGB videos using a pre-trained VGG-like network \cite{lea2016segmental}, and is downsampled to achieve the same temporal resolution as \cite{lea_2017_cvpr} for fair comparison.

\myparagraph{50 Salads \cite{stein2013combining}.} 
This multi-modal dataset collects 50 recordings from 25 people preparing 2 mixed salads, and each recording lasts 5-10 minutes. The RGB video has spatial resolution of 640x480 pixels and frame rate of 30 fps. The annotation is performed at two levels: (1) the {\em eval-level} incorporating 9 actions such as ``cut'', ``peel'' and ``add dressing'', and (2) the {\em mid-level} incorporating 17 fine-grained actions, derived from the high-level actions. 
Therefore, we obtain two sets from {\bf 50 Salads}, namely {\bf 50 Salads-eval} and {\bf 50 Salads-mid}. The recordings are equally split into 5 folds for cross-validation.

\myparagraph{Georgia Tech Egocentric Activity Datasets (GTEA) \cite{fathi2011learning} \cite{li2015delving}.}
This dataset contains 7 daily living activities of 4 subjects. The videos are captured from the egocentric view at 15 fps with the resolution of 1280x720 pixels and there are 31,222 frames in the dataset. We follow the settings in \cite{lea2016segmental} \cite{lea_2017_cvpr}: For each video, frame-wise labels from 11 action classes are annotated. The evaluation is based on the leave-one-subject-out scheme, namely performing cross-validation on 4-fold splits.

\myparagraph{JHU-ISI Gesture and Skill Assessment Working Set (JIGSAWS) \cite{gao2014jhu} \cite{ahmidi2017dataset}.} In our study we only use the videos of ``suturing'' since it has more trials than other tasks. The ``suturing'' task comprises 10 actions like ``tie a knot'', ``insert needle into skin'' and so forth. Each video is approx. 2 minutes and contains 15 to 37 actions, which have considerably different occurrence orders from different surgeons. Similar to GTEA, in our experiments we perform evaluations in the leave-one-surgeon-out scheme.

\subsection{Evaluation Metrics}

\myparagraph{Frame-wise accuracy.} 
The frame-wise accuracy is defined as the correctly classified frames divided by the number of all frames. Intuitively, such measure evaluates the accuracy from the frame-wise classification perspective. However, it ignores the temporal regularity and the action occurrence order in the label sequence.

\myparagraph{Edit score \cite{lea2016segmental}.}
The edit score evaluates the temporal order of action occurrence, ignores the action temporal durations and only considers segment insertions, deletions and substitutions. Thus, such metric is useful for scenarios, where the action order is essential, e.g. cooking, manufacturing, surgery and so forth. However, the edit score can be strongly penalized by tiny predicted segments, and hence highly degraded by over-segmentation results.

\myparagraph{F1 score \cite{lea_2017_cvpr}.} 
The F1 score is for evaluation in terms of action detection, where the true positives are defined by segments whose action label is same to the ground truth and the {\em intersection-over-union} of the overlap with the ground truth is greater than 0.1. Thus, it is invariant to small temporal shifts between detection and the ground truth. However, the F1-score is penalized by over-segmentation as well, since lots of tiny segments can result in a low precision rate.


\subsection{Analysis of the Bilinear Forms}
\label{sec:max_vs_bi}

We use the {\bf 50 Salads-mid} dataset to perform model analysis, because it has more fine-grained action types and longer video recordings than the other mentioned datasets. 

\begin{figure}[t]
  \centering
  \includegraphics[width=1.\linewidth]{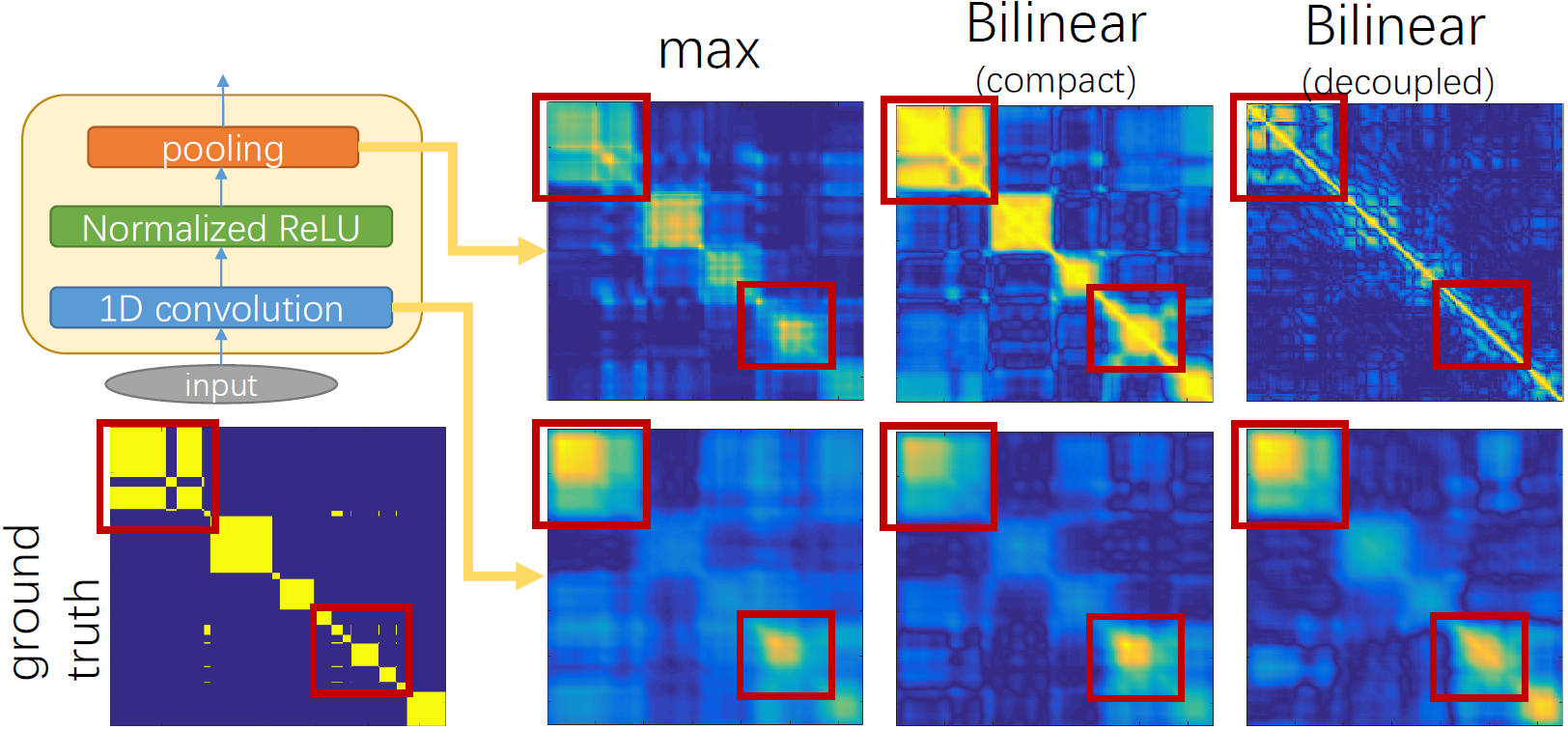}
  \caption{\small{We use the features from the first encoder of TCED to show frame similarities of ``rgb-01-1.avi'' in {\bf 50 Salads-mid} \cite{stein2013combining}. The similarity of two frame features ${\bm x}_i$ and ${\bm x}_j$ is defined as $| \langle {\bm x}_i, {\bm x}_j \rangle|$. The similarity of two (one-hot) frame labels, regarded as ground truth, is computed in the same way. The bilinear pooling outputs are power-and-l2 normalized. Entries in similarity matrices range between 0 (blue) and 1 (yellow). {\color{red} Red rectangles} contain some fine-grained actions. One can see that bilinear pooling is better at recognizing fine-grained actions, but can decompose coarse-grained actions.}}
  \vspace{-4mm}
 \label{fig:teaser}
\end{figure}

\myparagraph{The benefits of second-order information.}
Fig.~\ref{fig:teaser} illustrates the comparisons between pooling methods, where the compact bilinear pooling \cite{gao2016compact} output has the {\bf same} dimension as the max pooling:
{(1)} The first row in Fig.~\ref{fig:teaser} clearly shows that bilinear pooling can capture fine-grained actions better than max pooling, whose output features tend to merge fine-grained actions into coarse-grained ones. Our proposed decoupled bilinear pooling with full second-order information performs better at recognizing fine-grained actions and suppressing off-diagonal elements. However, it can break a coarse-grained action into several segments. 
The compact bilinear pooling outperforms max pooling on the diagonal elements, which can clearly show that the advantage of bilinear pooling is due to the second-order information rather than higher dimensionality. However, the large off-diagonal values indicate the drawback of the dimension reduction method with approximation.
{(2)} The second row in Fig.~\ref{fig:teaser} illustrates that bilinear pooling improves the convolution layer via backpropagation. With max pooling, many off-diagonal elements are similar to the diagonal elements, which differ from the ground truth pattern considerably. However, with bilinear pooling, the matrix patterns are more similar to the ground truth. 

Furthermore, we conduct a quantitative comparison on the {\em first split} of {\bf 50 Salads-mid}. In the format of {\em accuracy/edit-score/F1-score}, max pooling yields 71.03/71.8/73.09, compact bilinear pooling yields 75.41/73.75/78.96, coupled bilinear pooling yields 76.56/75.32/79.84 and decoupled bilinear pooling yields 75.11/71.06/75.79. One can see that bilinear pooling consistently outperforms max pooling. The comparison between max pooling and compact bilinear pooling also indicates the importance of the second-order information.

\myparagraph{Comparison of different bilinear forms.}
Here we analyze the influence of the learnable weights in the proposed bilinear forms $\mathcal{B}_c$ and $\mathcal{B}_d$.
We denote the corresponding non-learnable bilinear forms in Eq. \eqref{eq:bilinear_standard} and Eq. \eqref{eq:decoupled_bilinear} as $\mathcal{B}^{o}_c$ and $\mathcal{B}^{o}_d$, respectively.
As shown in the top row of Fig.~\ref{fig:ablation_1}, both for the coupled and decoupled bilinear forms, the one with learnable weights consistently outperforms the non-learnable counterpart, in terms of the evaluation metrics and the robustness to the neighborhood size $|\mathcal{N}|$. This outcome is more obvious when the neighbor size is larger. This result can indicate that the learnable weights, i.e. $\{\omega_{\tau}\}$, $\{p_{\tau}\}$ and $\{q_{\tau}\}$ in equations \eqref{eq:bilinear1} and \eqref{eq:bilinear2}, enable the derived bilinear forms to capture more complex local temporal statistics, comparing to the standard average aggregation. Thus, in the following experiments, we only use the learnable bilinear forms.
Furthermore, the decoupled bilinear form outperforms the coupled version on all the three metrics.
Specifically, the decoupled bilinear form achieves {\em 66.3/64.63/70.74} in the format of {\em accuracy/edit score/F1 score}, while the best performance of the coupled bilinear form is {\em 64.73/62.15/68.89} and the baseline model (TCN$_{max}$ \cite{lea_2017_cvpr}) achieves {\em 64.7/59.8/68.0}.

In the bottom row of Fig.~\ref{fig:ablation_1}, we show the performance of the first-order component and the second-order component of the decoupled bilinear form. One can observe that 
the results derived using individual components are inferior to the results using combined bilinear forms. This fits our conjecture that first and second-order components tend to describe independent and complementary patterns in data.  


\begin{figure}[t!]
\centering
\includegraphics[width=0.485\textwidth]{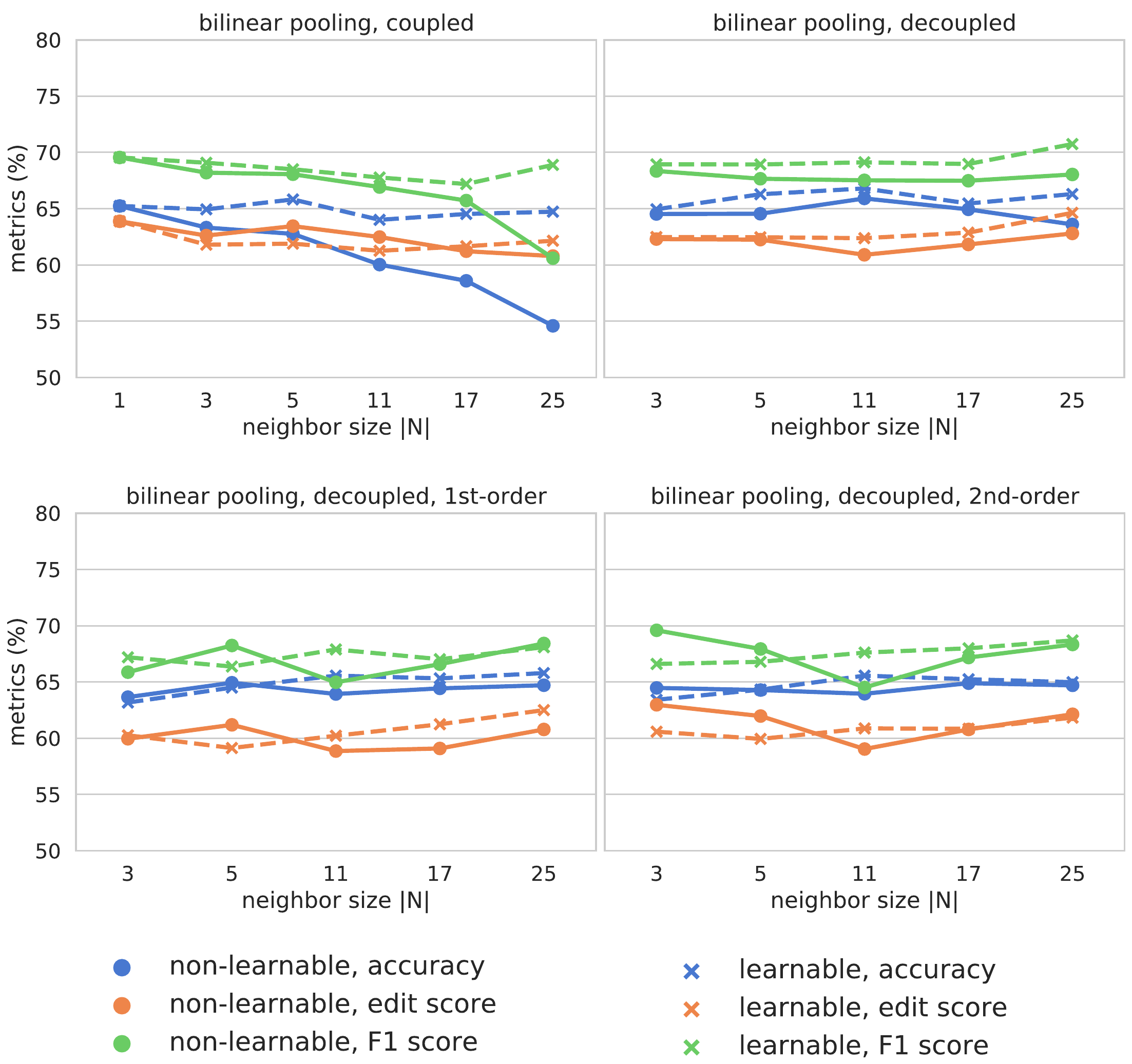}
\caption{The performances w.r.t. the neighbor size $|\mathcal{N}|$ and the learnability of the weights. From top to bottom: (1) The performances of the coupled bilinear form $\mathcal{B}_c$ and the decoupled bilinear form $\mathcal{B}_d$. (2) The performances of each ingredient in the decoupled bilinear form $\mathcal{B}_d$, in which the first-order component and the second-order component are demonstrated in Eq. \ref{eq:bilinear2}.}
\label{fig:ablation_1}
\end{figure}

%
%

\myparagraph{Normalization and activation.}
Here we investigate the influence of normalization and compare different activation functions. In each individual experiment the neighborhood size of both bilinear forms are identical.
First, different normalization methods are compared in Tab.~\ref{tab:norm}. One can see that $l_2$ normalization and $l_1$ normalization perform almost equally, while the normalized ReLU activation function consistently outperforms others. This result indicate that the max-normalization in intermediate layers is more suitable than others to constrain the bilinear form spectrum. 
Second, we show the influence of the activation functions in Tab.~\ref{tab:actfun}. The bilinear forms are $l_2$ normalized, except for the case of {\bf NReLU}. In our experiment, training with other activation functions without $l_2$ normalization hardly converges, indicating the importance of constraining the spectrum of the bilinear form. Tab.~\ref{tab:actfun} indicates that the {\bf NReLU} function consistently yields superior results, suggesting that our task benefits from  the sparse features.

\begin{table}[h!]
\centering
\small%
\begin{tabular}{l|ll}

& $\mathcal{B}_c$ & $\mathcal{B}_d$ \\
\hline
{\bf NReLU} & {\bf 65.82}/{\bf 61.89}/{\bf 68.5} & {\bf 66.28}/{\bf 62.46}/{\bf 68.93} \\
{\bf NReLU}+$l_2$ & 64.92/60.01/67.33 & 66.09/60.02/67.38 \\
{\bf NReLU}+$l_1$ & 64.22/60.04/65.86 & 66.48/63.04/68.71\\
{\bf ReLU}+$l_2$ & 64.87/59.88/66.31 & 64.7/61.45/68.04 \\
{\bf ReLU}+$l_1$ & 63.05/58.29/65.62 & 59.76/58.68/64.39 \\ 
 \hline
\end{tabular}
\caption{Comparison of different normalization methods, in which the performances are presented in terms of {\em accuracy/edit score/F1 score} and the best ones for each model are highlighted in boldface.}
\label{tab:norm}

\end{table}

\begin{table*}
\centering
\small
\begin{tabular}{l|llllll}

& {\bf ReLU} \cite{nair2010rectified} & {\bf leaky ReLU} \cite{maas2013rectifier} & {\bf swish} \cite{swish} & {\bf NReLU (Eq. \ref{eq:nrelu})} \cite{lea_2017_cvpr} & {\bf RPN (Eq. \ref{eq:rpn})} & {\bf linear} \\

\hline
max & 61.13/53.13/59.78 & 54.97/48.51/55.58 & 56.51/47.06/52.39 & {\bf 63.55}/{\bf 60.37}/{\bf 64.88} & 62.65/54.89/63.05 & 12.59/11.51/8.63\\

$\mathcal{B}_c$ & 65.5/61.14/68.4 & 66.4/{\bf 61.72}/{\bf 69.08} & {\bf 66.77}/59.16/67.49 & 64.01/61.26/67.77 & 64.05/56.48/64.77 & 63.47/48.34/55.87\\

$\mathcal{B}_d$ & 64.7/61.45/68.04 & 65.56/53.55/61.45 & 62.51/49.26/56.64 & {\bf 66.8}/{\bf 62.38}/{\bf 69.12} & 63.18/50.41/58.39 & 65.51/48.31/56.96\\

\hline

\end{tabular}
\caption{The performances with different pooling methods and activation functions are presented in the format of {\em accuracy/edit score/F1 score}, in which for each model the best results are highlighted in boldface. }
\label{tab:actfun}

\end{table*}

%
%
%
%
%

\subsection{Low-dimensional Representation}
One of our key contributions is to derive lower-dimensional alternatives to the explicit bilinear compositions. Comparing other dimension reduction methods, our method does not suffer from any information loss nor do we have extra computation. We compare different low-dimensional representations in the lower parts of Table \ref{tab:compare_50salads_mid}, \ref{tab:compare_50salads_eval}, \ref{tab:compare_gtea} and \ref{tab:compare_jigsaws}. The {\em tensor sketch} technique \cite{gao2016compact} reduces each feature outer product from $d^2$ to $\frac{d(d+1)}{2}$ for fair comparison. In addition, the {\em LearnableProjection} \cite{lin2018bilinear} is implemented by a temporal convolution layer with the kernel size of 1, and the reduced dimensions are equal to $\phi_c$ and $\phi_d$ respectively for fair comparison. Note that, in our trials, other dimension reduction methods (especially the ones employing SVD) used in our local temporal pooling cause very high computational cost, and hence are not compared. For each listed method we tested different neighborhood sizes of 5, 11 and 25, and present the best performance. Our results show that the proposed low-dimensional representations consistently outperform other dimension reduction methods. In particular, on the {\bf 50 Salads-mid} dataset, $\phi_d$ considerablely outperforms the {\em LearnableProjection} counterpart, in which the accuracy is improved by $5.6\%$, the edit score is improved by $6.2\%$ and the F1 score is improved by $6.1\%$.

\vspace{-4mm}

\subsection{Comparison with State-of-the-art}

Table \ref{tab:compare_50salads_mid}, \ref{tab:compare_50salads_eval}, \ref{tab:compare_gtea} and \ref{tab:compare_jigsaws} show the performances of different methods on the datasets {\bf 50 Salads-mid}, {\bf 50 Salads-eval}, {\bf GTEA} and {\bf JIGSAWS}, respectively, in which TCED$_{X}$ denotes the temporal convolutional encoder-decoder with the pooling method $X$. 
For each method with local temporal pooling, we perform grid search on the neighborhood sets of 5, 11 and 25, and present the best one. From the tables, we can see that our proposed method can be generalized well across different datasets and produces superior or comparable performances than other methods. In {\bf 50 Salads-mid}, the dataset with more fine-grained action types and longer videos than other datasets, the decoupled bilinear form, as well as its lower-dimensional representation outperform other methods for all the evaluation metrics. In {\bf 50 Salads-eval}, the performance of our methods are comparable with others while with lower edit scores, probably because actions in this dataset is not sufficiently fine-grained but our bilinear pooling produces more segments than others. Furthermore, more training epochs can increase the accuracy yet decrease the edit score and the F1 score for our bilinear pooling models, in contrast to the max pooling baseline model. For example, after 300 epochs, TCED$_{\mathcal{B}_d}$ yields 74.7/59.2/66.7 and TCED$_{max}$ yields 63.6/71.9/75.2 for the {\bf GTEA} dataset.

\begin{table}
\centering
\begin{tabular}{l|l}

{\bf Method} & {\bf Result}\\
\hline
\hline
Spatial CNN \cite{lea2016segmental} & 54.9/24.8/32.3 \\
Spatiotemporal CNN \cite{lea2016segmental} & 59.4/45.9/55.9 \\
IDT+LM \cite{richard2016temporal} & 48.7/45.8/44.4 \\
Dilated TCN \cite{lea_2017_cvpr} & 59.3/43.1/52.2 \\
Bidirectional LSTM \cite{lea_2017_cvpr} & 55.7/55.6/62.6 \\
TCED$_{max}$ \cite{lea_2017_cvpr} & 64.7/59.8/68.0 \\
\hline
TCED$_{\mathcal{B}_c}$ & 65.8/61.9/68.5\\
TCED$_{\mathcal{B}_d}$ & {\bf 66.3}/{62.5}/{68.9}\\
\hline
\hline
TCED$_{TensorSketch}$ \cite{gao2016compact} & 63.4/62.6/68.5 \\
TCED$_{\mathcal{B}_c, LearnableProjection}$ & 61.8/58.2/64.4 \\ 
TCED$_{\mathcal{B}_d, LearnableProjection}$ & 60.1/56.6/62.9 \\ 
\hline
TCED$_{\phi_c}$ & 64.7/61.3/66.8 \\ 
TCED$_{\phi_d}$ & 65.7/{\bf 62.8}/{\bf 69.0} \\ 
\hline
\end{tabular}
\caption{The comparison in {\bf 50 Salads-mid}, where the results are shown in the format of {\em accuracy/edit score/F1 score}. The upper part shows the comparison with other action parsing methods and the lower part shows the comparison of different dimension reduction methods. The best results are highlighted in boldface.}
\label{tab:compare_50salads_mid}

\end{table}

\begin{table}
\centering
\begin{tabular}{l|l}

{\bf Method} & {\bf Result}\\
\hline
\hline
Spatial CNN \cite{lea2016segmental} & 68.0/25.5/35.0 \\
Spatiotemporal CNN \cite{lea2016segmental} & 71.3/52.8/61.7 \\
Dilated TCN \cite{lea_2017_cvpr} & 71.1/46.9/55.8 \\
Bidirectional LSTM \cite{lea_2017_cvpr} & 70.9/67.7/72.2 \\
TCED$_{max}$ \cite{lea_2017_cvpr} & 73.4/{\bf 72.2}/{\bf 76.5} \\
\hline
TCED$_{\mathcal{B}_c}$ & 74.2/71.2/75.5\\
TCED$_{\mathcal{B}_d}$ & {\bf 75.9}/71.3/76.2\\
\hline
\hline 
TCED$_{TensorSketch}$ \cite{gao2016compact} & 71.9/70.9/75.1 \\
TCED$_{\mathcal{B}_c, LearnableProjection}$ & 72.0/68.8/73.4 \\ 
TCED$_{\mathcal{B}_d, LearnableProjection}$ & 71.3/68.9/72.6\\ 
\hline
TCED$_{\phi_c}$ & 74.0/71.0/{\bf 76.5} \\ 

TCED$_{\phi_d}$ & 75.6/70.4/76.0 \\ 
\hline
\end{tabular}
\caption{The comparison in {\bf 50 Salads-eval}.}
\label{tab:compare_50salads_eval}

\end{table}

\begin{table}
\centering
\begin{tabular}{l|l}

{\bf Method} & {\bf Result}\\
\hline \hline
EgoNet+TDD \cite{singh2016first} & { 64.4}/-/- \\
Spatial CNN \cite{lea2016segmental} & 54.8/28.7/38.3 \\
Spatiotemporal CNN \cite{lea2016segmental} & 57.6/49.1/56.7 \\
Spatiotemporal CNN+Seg \cite{lea2016segmental} & 52.6/53.0/57.7 \\
Dilated TCN \cite{lea_2017_cvpr} & 58.0/40.7/51.3 \\
Bidirectional LSTM \cite{lea_2017_cvpr} & 56.2/41.3/50.2 \\
TCED$_{max}$ \cite{lea_2017_cvpr} & 63.5/{71.9}/75.2 \\
\hline
TCED$_{\mathcal{B}_c}$ & 63.6/71.7/76.4\\
TCED$_{\mathcal{B}_d}$ & 63.4/70.9/{\bf 76.8}\\
\hline
\hline
TCED$_{TensorSketch}$ \cite{gao2016compact} & 59.8/71.2/75.2 \\
TCED$_{\mathcal{B}_c, LearnableProjection}$ & 58.4/68.2/71.9\\ 
TCED$_{\mathcal{B}_d, LearnableProjection}$ & 58.8/70.5/74.9\\
\hline
TCED$_{\phi_c}$ & {\bf 64.5}/71.8/75.0 \\ 
TCED$_{\phi_d}$ & 64.4/{\bf 73.9}/76.3 \\ 
\hline
\end{tabular}
\caption{The comparison in {\bf GTEA}, in which the symbol ``-'' denotes that the score is not available. }
\label{tab:compare_gtea}
\end{table}

\begin{table}
\centering
\begin{tabular}{l|l}

{\bf Method} & {\bf Result}\\
\hline \hline
Spatial CNN \cite{lea2016segmental}& 74.1/37.7/51.6\\
Spatiotemporal CNN \cite{lea2016segmental} & 77.9/67.1/77.7 \\
Spatiotemporal CNN+Seg \cite{lea2016segmental} & 74.4/73.7/82.2 \\
Dilated TCN \cite{lea_2017_cvpr} & 78.0/56.8/69.7 \\
Bidirectional LSTM \cite{lea_2017_cvpr} & 74.4/73.7/82.2 \\
TCED$_{max}$ \cite{lea_2017_cvpr} & 81.2/85.6/90.3 \\
\hline
TCED$_{\mathcal{B}_c}$ & {\bf 82.6}/85.6/90.4\\
TCED$_{\mathcal{B}_d}$ & 82.2/{\bf 87.7}/{\bf 91.4}\\
\hline 
\hline 
TCED$_{TensorSketch}$ \cite{gao2016compact} & 80.8/85.4/90.1 \\
TCED$_{\mathcal{B}_c, LearnableProjection}$ & 79.7/82.8/88.1\\ 
TCED$_{\mathcal{B}_d, LearnableProjection}$ & 81.6/83.0/89.0\\
\hline
TCED$_{\phi_c}$ & 81.8/85.1/90.0 \\ 
TCED$_{\phi_d}$ & 81.7/85.1/90.5 \\ 
\hline 
\end{tabular}
\caption{The comparison in {\bf JIGSAWS}. }
\label{tab:compare_jigsaws}

\end{table}

\section{Conclusion}
\label{sec:conclusion}
To our knowledge, we are the first to use bilinear pooling to a temporal convolutional encoder-decoder for action parsing. To enrich representativeness, we decouple the first and the second-order information from the conventional bilinear form and modify the averaging operation to convolution with a learnable filter. To reduce dimensionality, we introduce lower-dimensional representations of the bilinear forms with neither information loss nor extra computation. We conduct several detailed experiments to analyze the bilinear forms, and show superior performances to state-of-the-art pooling methods for action parsing. A future work is to investigate higher-order pooling with information lossless dimension reduction approaches.

\noindent{\bf Acknowledgements.}
Y. Z. and H. N. acknowledge funding by the BMBF project SenseEmotion. 
S. T. acknowledges funding by Deutsche Forschungsgemeinschaft (DFG, German Research Foundation) – Projektnummer 276693517 – SFB 1233.
We faithfully acknowledge Dr. Colin Lea (Facebook) to provide frame-wise features of the datasets.

{\small
\bibliographystyle{unsrt}
\bibliography{reference_merge}
}

\end{document}